\documentclass[journal]{IEEEtran} 

\makeatletter

\let\proof\@undefined
\let\endproof\@undefined
\makeatother

\IEEEoverridecommandlockouts               

\usepackage{amsmath,amsthm} 
\usepackage{amssymb} 
\usepackage{mathtools}
\usepackage{graphicx}
\usepackage{calc}
\usepackage{tabularx} 
\usepackage{graphicx}
\usepackage{subcaption}
\usepackage{url}
\usepackage{booktabs}
\usepackage{multicol}
\usepackage{enumitem,kantlipsum}
\usepackage{float}
\usepackage{textcomp}
\usepackage{graphics}
\usepackage{tabulary}
\usepackage{lipsum,caption,graphicx}
\usepackage{mathtools}

\newtheorem{theorem}{Theorem}[section]

\newtheorem{observation}[theorem]{Observation}
\newtheorem{lemma}[theorem]{Lemma}

\theoremstyle{definition}
\newtheorem{definition}[theorem]{Definition}
\newtheorem{problem}[theorem]{Problem}

\theoremstyle{remark}
\newtheorem{remark}[theorem]{Remark}

\renewcommand{\baselinestretch}{0.975}
\usepackage{algorithm}
\usepackage[noend]{algpseudocode} 
\newcommand{\s}{\mathcal{S}}
\newcommand{\X}{\mathcal{X}}
\newcommand{\ChiObs}{\mathcal{X}_{\text{obs}}}

\usepackage[dvipsnames]{xcolor}
\newcommand{\blue}[1]{\textcolor{black}{#1}}
\usepackage{cite}
\usepackage{graphicx, nicefrac}
\title{\LARGE \bf
Minimal $t-$spanning Primitive Sets: An MILP Formulation
}

\title{\LARGE \bf
Spatio-Temporal Lattice Planning Using Optimal Motion Primitives}

\urldef{\smith}\url{stephen.smith@uwaterloo.ca}
\urldef{\botros}\url{alexander.botros@uwaterloo.ca}

\author{Alexander Botros and Stephen L.\ Smith
\thanks{This work is supported in part by the Natural Sciences and Engineering Research Council of Canada (NSERC)}
\thanks{The authors are with the Department of Electrical and Computer Engineering,
    University of Waterloo, 200 University Ave W, Waterloo ON, Canada, N2L 3G1
  (\botros, \smith)
    }
}

\begin{document}

\maketitle
\thispagestyle{empty}
\pagestyle{empty}

\begin{abstract}
Lattice-based planning techniques simplify the motion planning problem for autonomous vehicles by limiting available motions to a pre-computed set of \emph{primitives}. These primitives are combined online to generate complex maneuvers. A set of motion primitives $t$-span a lattice if, given a real number $t\geq 1$, any configuration in the lattice can be reached via a sequence of motion primitives whose cost is no more than a factor of $t$ from optimal. Computing a minimal $t$-spanning set balances a trade-off between computed motion quality and motion planning performance. In this work, we formulate this problem for an arbitrary lattice as a mixed integer linear program. We also propose an A*-based algorithm to solve the motion planning problem using these primitives and an algorithm that removes the excessive oscillations from planned motions -- a common problem in lattice-based planning. Our method is validated for autonomous driving in both parking lot and highway scenarios.
\end{abstract}

\section{Introduction}
Broadly, the motion planning problem is to search the configuration space of a system for a collision free path between a given start and goal \cite{844730} that optimizes desirable properties like travel time or comfort~\cite{paden2016survey}. This path or \emph{motion} can be used as a reference for a tracking controller~\cite{urmson2008autonomous} for a fixed amount of time before a new motion is planned.

In general, the motion planning problem---the focus of this work---is intractable \cite{8618964} owing in part to potentially complex vehicle kinodynamic constraints that impede the calculation of two point boundary value (TPBV) problem solutions. Thus, simplifying assumptions are made. The variety in possible simplifications has given rise to several planning techniques. These typically fall into one of four categories \cite{gonzalez2015review}: sampling based planners, interpolating curve planners, numerical optimization approaches, and graph search based planners.

Lattice-based motion planning is an example of a graph search planner~\cite{gonzalez2015review} and is one of the most common approaches to solving the motion planning problem\cite{gonzalez2015review, tiger2021enhancing}.
It works by discretizing the vehicle's configuration space into a countable set (or \emph{lattice}) of regularly repeating configurations. Kinodynamically feasible motions between lattice configurations (or \emph{vertices}) are pre-computed and a subset of these motions called a \emph{control set} is selected~\cite{ tiger2021enhancing, pivtoraiko2005generating, pivtoraiko2009differentially, oliveira2018combining}. Elements of this control set (\blue{called \emph{motion primitives}}) can be concatenated in real time to form complex maneuvers \blue{as in Figure \ref{TspanEx}, where motion primitives (magenta) are concatenated end-to-end to produce a final motion (red)}. 

Lattice-based planners simplify the motion planning problem by limiting the set of available motions instead of simplifying their kinodynamics -- guaranteeing feasibility. \blue{Appropriately selected control sets can yield theoretical guarantees on sub-optimality with respect to free-space optimal while keeping the size of the search graph small~\cite{janson2018deterministic}}. These guarantees do not require the fineness of the lattice grid structure to approach infinity, which is typically required in sampling-based techniques. \blue{Lattice-based motion planning is particularly attractive in scenarios with complex kinodynamic constraints since computationally expensive TPBV problems between lattice vertices are solved offline.} 

Though the process of selecting an appropriate resolution for the lattice discretization is inarguably important, this work focuses primarily on the traversal of the lattice once it has been created. \blue{Further, we do not propose a method for solving TPBV problems between lattice vertices given kinematic constraints, preferring instead to keep the methods proposed herein general and applicable to any motion constraints.} Despite its many advantages, lattice-based motion planning is not without its critiques. This work presents novel solutions for three such critiques:

\begin{figure}[t]
\centering

 \includegraphics[width = 0.94\columnwidth]{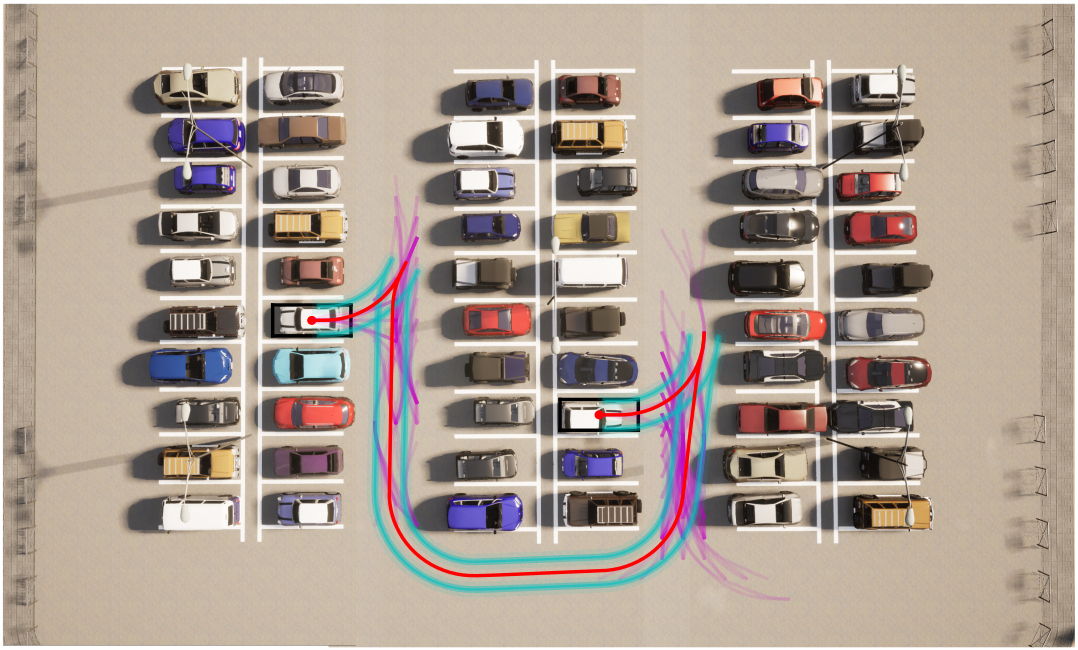}
 \caption{Motion planning using $t$-spanning $G3$ motion primitives. Magenta: primitives used, Red: final motion, Cyan: car footprint.}
 \label{TspanEx}
 
\end{figure}
\subsubsection{Choice of Control Set} It is observed in \cite{pivtoraiko2011kinodynamic, botros2019computing} that the \emph{number} of primitives in a control set favorably affects the quality of the resulting paths but adversely affects the run-time of an online search. The authors of \cite{pivtoraiko2011kinodynamic} introduce the notion of a $t$-spanning control set: a control set that guarantees the existence of compound motions -- motions formed by concatenating motion primitives -- whose costs are no more than a factor of $t$ from optimal. They propose to find the smallest set that, for a given value $t\geq 1$, will $t$-span a lattice thus optimizing a trade-off between path quality and performance. Computing such a control set is called the Minimum $t$-Spanning Control Set (MTSCS) problem which is known to be $NP$-hard \cite{botros2019computing}. \blue{Hitherto, the only known exact solver for the MTSCS problem involves a brute-force approach where each potential control set is considered~\cite{pivtoraiko2011kinodynamic}}.
\subsubsection{Continuity in All States \& On-lattice Propagation} Given a kinodynamically feasible motion originating from an initial configuration $p_s$, a new curve can be obtained from a configuration $p_s'$ by rotating and translating the original curve if $p_s$ and $p_s'$ differ only in position and orientation. However, this is no longer true if $p_s$ and $p_s'$ differ in higher order states like curvature. For this reason, position and orientation are called variant states while higher order states are invariant \cite{pivtoraiko2011kinodynamic}. When motion primitives are concatenated, invariant states may experience a discontinuous jump if all primitives are computed relative to a single starting vertex \cite{pivtoraiko2011kinodynamic}. Further, for lattices with non-cardinal headings, using a single starting vertex to define all motion primitives may result in off-lattice endpoints during concatenations, requiring rounding and resulting in sub-optimal motions \cite{rusu2014state, dolgov2010path}. A solution approach to this critique that we employ is to compute a control set for a lattice with several starting vertices \cite{pivtoraiko2011kinodynamic, tiger2021enhancing}. 
\subsubsection{Excessive Oscillations} Because the set of available motions is restricted in lattice-based motion planning, compound motions comprised of several primitives may possess excessive oscillations \cite{oliveira2018combining} resulting in motions which are kinodynamically feasible but zig-zag unnecessarily.

\subsection{Contributions}
The solutions proposed in this paper to the challenges above are summarized here.
\begin{itemize}[leftmargin=*]
  \item \textit{Choice of control set:} We propose a mixed integer linear programming (MILP) encoding of the MTSCS problem for state lattices with several starting vertices. This represents the first known non-brute force approach to solving the MTSCS problem. Though this formulation does not scale for large lattices, we observe that control sets need only be computed once, offline.
  \item \textit{Continuity in All States:} The solution to the proposed MILP is the smallest control set that generates motions that are continuous in all states with bounded $t$-factor sub-optimality. 
  \item \textit{Reduction of Excessive Oscillations:} We provide a novel algorithm that eliminates redundant vertices along motions computed using a lattice. This algorithm is based on shortest paths in directed acyclic graphs and eliminates excessive oscillations. It runs in time quadratic in the number of motion primitives along the input motion.
   \item \textit{Lattice Search Algorithm:} We present an A*-based algorithm to compute feasible motions for difficult maneuvers in both parking lot and highway settings. The algorithm accommodates off-lattice start and goal configurations up to a specified tolerance.
\end{itemize}
In our preliminary work~\cite{botros2019computing}, we introduced an MILP formulation for the MTSCS problem for lattices with a single starting vertex and proved the NP-completeness of this sub-problem. The novel contributions of this paper over our preliminary work include a MILP formulation of the MTSCS problem for a general lattice featuring several starting vertices, an A*-based search algorithm for use in lattice-based motion planning, and a post-processing smoothing algorithm to remove excessive oscillations from motions planned using a lattice. 

\subsection{Related Work}
 In this section, we review the four categories of planning techniques from \cite{gonzalez2015review}: sampling based, interpolating curve, numerical optimization, and graph search based (in particular lattice-based motion planning). We will outline key advantages to lattice-based motion planning as well as some of the deficiencies in the area that have motivated our contributions.
 
 Sampling based planners work by randomly sampling configurations in a configuration space and checking connectivity to previous samples. Common examples include Asymptotically Optimal Rapidly-exploring Random Trees (RRT*) \cite{karaman2011sampling}, and Probabilistic Roadmap (PRM) \cite{kavraki1996probabilistic}. These methods use a local planner to expand or re-wire a tree to include new samples. Therefore, many TPBV problems must be solved online necessitating the use of simplified kinodynamics. Kinodynamic feasibility may not be guaranteed \cite{gonzalez2015review}. Sampling-based planners, while not complete, can be \emph{asymptotically complete} with a convergence rate that worsens as the dimensionality of the problem is increased \cite{solovey2020revisiting}. To improve the performance of RRT* in higher dimensions, the authors of \cite{pendleton2017numerical} pre-compute a set of reachable configurations from which random samples are drawn. The pre-computed set is a lattice in which motions are computed by random sampling. Though appropriate for autonomous driving in unstructured environments like parking lots, sampling-based techniques are not widely used in highly structured environments like highways~\cite{claussmann2019review}. This paper proposes a planner that is appropriate in both settings.  

Techniques using the interpolating curve approach include fitting Bezier curves, Clothoid curves \cite{fuji2014trajectory}, or polynomial splines \cite{inproceedings} to a sequence of way-points. A typical criticism of clothoid-based interpolation techniques is the time required to solve TPBV problems involving Fresnel integrals \cite{gonzalez2015review, ravankar2018path}. While TPBV problems are typically easily solved in the case of polynomial spline and Bezier curve interpolation, it is often difficult to impose constraints like bounded curvature on these curves owing to their low malleability \cite{gonzalez2015review, ravankar2018path}. In \cite{schmerling2015optimal}, the authors develop a sampling-based planner that uses a control-affine dynamic model to facilitate solving TPBV problems. These motions are then smoothed via numeric optimization in \cite{zhu2015convex}. Here, motions with non-affine non-holonomic constraints and piece-wise linear velocity profiles are developed. The authors of \cite{zhu2015convex} illustrate the benefits of computing way-points using a system of similar complexity to the desired final motion. However, by the very nature of the simplification, it is possible for infeasible way-point paths to be developed. A further common critique of both interpolating curve and numerical optimization approaches is their reliance on global way-points \cite{gonzalez2015review, ravankar2018path}.

In contrast, lattice-based planners with motion primitives are capable of planning feasible trajectories \cite{dolgov2008practical} that do not rely on way-points \cite{gonzalez2015review}. In \cite{dolgov2008practical, rusu2014state}, the authors use motion primitives to traverse a lattice with states given by position and heading. In both these works, compound motions comprised of several primitives are discontinuous in curvature -- a known source of slip and discomfort \cite{levinson2011towards}. To improve smoothness, the authors of \cite{9062306} use the techniques of \cite{dolgov2008practical} to compute an initial trajectory that is then smoothed via numerical optimization. Similarly, in \cite{zhang2018autonomous}, paths from \cite{dolgov2008practical} are used as warm starts for an optimization-based collision avoidance algorithm for use is autonomous parking. The methods in \cite{9062306, zhang2018autonomous} require at least as much computation time as the methods proposed in \cite{dolgov2008practical}. In this work, we propose methods appropriate for both parking and highway driving scenarios that outperform the methods of \cite{dolgov2008practical} in both runtime and comfort metrics.

The choice of control set is of particular interest. Typically motion primitives are chosen to achieve certain objectives for the paths they generate. The authors of \cite{zhang2018hybrid} maximize comfort by computing motion primitives that minimize the integral of the squared jerk over the motion. However, they limit their results to forward motion limiting applicability in parking lot scenarios. Natural behavior is the objective in \cite{paraschos2013probabilistic} which introduces probabilistic motion primitives to achieve a blending of deterministic motion primitives. In this work, we illustrate a method of computing a control set of motion primitives given \emph{any} objective. In \cite{oliveira2018combining} the authors consider the control set to be the entirety of the lattice which may necessitate \blue{coarser} lattices for the sake of time efficiency. The notion that some motions in a lattice are sufficiently similar to others to not be included in a control set motivates the work in \cite{jarin2021dispersion} which introduces a method of computing a control set is using the dispersion minimizing algorithm from \cite{palmieri2019dispertio}.

In \cite{pivtoraiko2005generating}, on the other hand, the authors present the notion of using a MTSCS of motion primitives -- a control set that balances a trade-off between \blue{its size and the optimality of resulting paths}. These are similar to graph $t$-spanners first proposed in~\cite{peleg1989graph}. The process of computing such a set is elaborated in \cite{botros2020learning} where we compute motions between lattice vertices that minimize a user-specified cost function that is learned from demonstrations. The methods presented herein may be used in tandem with other work \cite{botros2020learning, de2019learning} to provide an improved set of user-specific motion primitives.

In~\cite{pivtoraiko2011kinodynamic}, the authors present a heuristic for the MTSCS problem which, though computationally efficient, does not have any known sub-optimality factor guarantees. Since a control set may be computed once, offline, and used over many motion planning problems, the time required to compute this control set is of less importance than its size.

\section{Lattice Planning \& \blue{Problem Statement}}
\blue{We begin with a review of lattice-based motion planning and outline the problems with the approach that are addressed in this paper}.

Let $\X$ denote the configuration space of a vehicle. That is, $\X$ is a set of tuples --- called \emph{configurations} --- whose entries are the \emph{states} of the vehicle. A \emph{motion} from an initial configuration $p_s\in\X$ to a goal $p_g\in\X$ is a path in $\X$ beginning at $p_s$ and terminating at $p_g$. Let $\mathcal{M}$ denote the set of all kinodynamically feasible motions in $\X$ -- that is, the motions that adhere to a known kinodynamic model. We assume that each motion in $\mathcal{M}$ can be evaluated using a known cost function $c:\mathcal{M}\rightarrow \mathbb{R}_{\geq 0}$. Though $c$ is left general, we do assume that it obeys the triangle inequality. The motion planning problem is as follows.
\begin{problem}[Motion planning problem (MPP)]\label{MPP} Given $\X$, a set of obstacles $\ChiObs$, start and goal configurations $p_s, p_g\in \X 
 \backslash   \ChiObs$, and a cost function $c$, compute a motion from $p_s$ to $p_g$ that is feasible and optimal:
\begin{enumerate}
  \item \textbf{Feasible motions:} Each configuration in the motion is in $\X  \backslash\ChiObs$, and the motion is kinodynamically feasible.
  \item \textbf{Optimal motions:} The motion minimizes $c$ over all feasible motions from $p_s$ to $p_g$.
\end{enumerate}
\end{problem}
Lattice-based planning approximates a solution to this problem by regularly discretizing $\X$ using a \emph{lattice}, $L\subseteq\X$. Let $i, j$ be vertices in $L$, and let $p$ be a motion that solves Problem \ref{MPP} for $p_s=i, p_g=j$, written $i\cdot p = j$. The key observation in lattice-based motion planning is that $p$ may be applied to other vertices in $L$ potentially resulting in another vertex in $L$. \blue{This motivates the use of a starting vertex $s\in L$, usually at the origin, that represents a \emph{prototypical} lattice vertex. Offline, a feasible optimal motion is computed from $s$ to all other vertices $i\in L-\{s\}$ in the absence of obstacles. A subset $E$ \blue{--called a \emph{control set} --} of these motions is selected and used as an action set during an online search. The motions available at any iteration of the online search are isometric to those available at $s$ (called \emph{motion primitives})}. \blue{In this work, we address three elements of lattice-based motion planning:
    \paragraph{\textbf{Continuity In All States}}  Care must be taken to ensure that states vary continuously over concatenations of motions. For example, suppose velocity is a state, starting vertex $s$ is located at the origin (\emph{i.e.,} the velocity at $s$ is 0) and $p$ is a motion from $s$ to $i\in L$. If $i$ does not have 0 velocity, then no motion originating at $s$ can be applied to $i$ without incurring a discontinuity in velocity. To ensure state continuity, we employ a \emph{multi-start} lattice described in further detail in Section \ref{sec:MSL}.
    \paragraph{\textbf{Control Set Selection}} As discussed in the Introduction, the choice of control set influences both the performance of an online search of the resulting graph as well as the quality of the motions computed. In Section \ref{sec:Choosing_CS}, we formulate the problem of selecting a control set as a MTSCS problem in the context of a multi-start lattice. A mixed integer linear encoding of the problem is then proposed in Section \ref{sec:MILP}. 
    \paragraph{\textbf{Graph Traversal}} Motion primitives induce a search graph whose vertices are lattice vertices and whose edges are motions between vertices. In Section \ref{sec:AstarVar}, we present a variant of the A* algorithm for multi-start lattices. Further, in Section \ref{sec:Smoothing}, we propose a smoothing algorithm to remove excessive oscillations which often arise in lattice-based motion planning.
}
\begin{figure}[t]
\centering
  \includegraphics[width = 0.75\linewidth]{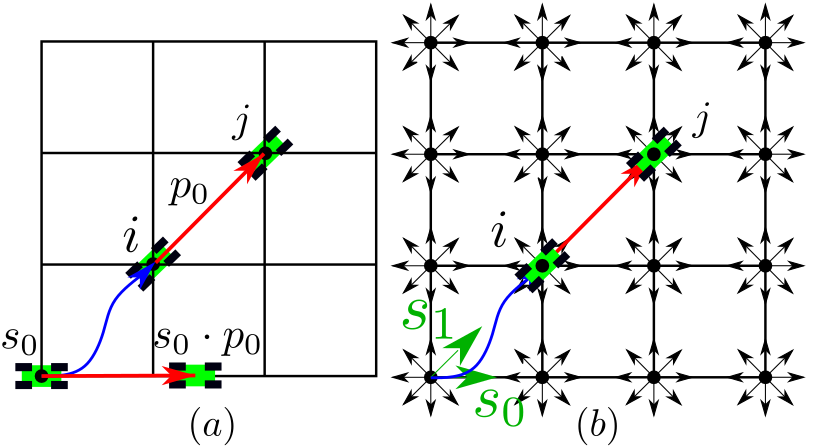}
  \caption{(a) Config. space, lattice and start from \eqref{BadX}. (b) Lattice and $\X$ as in (a) with start set $\s = [s_0=(0, 0, 0), s_1=(0, 0, \pi/4)]$. Blue motion from $s_0$ to $i=(1, 1, \pi/4)$, red from $i$ to $j=(2, 2, \pi/4)$.}
  \label{MultiStart}
\end{figure}

\section{Multi-Start Lattices}
\label{sec:MSL}
 There are examples for which a single start is insufficient to capture the full variety of motions \blue{if motions are constrained to be continuous in all states}. For example, let
\begin{equation}
\label{BadX}
\begin{split}
\X = &\mathbb{R}^2\times[0, 2\pi],\
L = \mathbb{Z}^2\times \{i\pi/4, i=0,\dots,7\},\\
s_0 = &(0, 0, 0),
\end{split}
\end{equation}
 denote configurations space, lattice, and starting vertex, respectively -- as in Figure \ref{MultiStart} (a). If $i=(1, 1, \pi/4)\in L$ and $j=(2, 2, \pi/4)\in L$, then the motion $p_0$ from $i$ to $j$ is such that $s_0\cdot p_0=(\sqrt{2}, 0, 0)\notin L$ \blue{if states $x,y,\theta$ vary continuously (as illustrated in Figure \ref{MultiStart} (a)). Therefore, $p_0$ is not an action available at $s_0$ implying that the simple diagonal motion $p_0$ will not appear in any action set $E$. }

We propose a solution to this problem by way of a multi-start lattice. Given a set of starts $\s\subset L$, vertices $i,j\in L$ and motion $p$ from $i$ to $j$, we say that $i\cdot p$ is a \emph{valid concatenation}, and write $i\oplus p$ if there exists a start $s\in \s$ such that $s\cdot p \in L$. In the example in Figure \ref{MultiStart} (b), let $\s=[s_0=(0, 0, 0), s_1 = (0, 0, \pi/4)]$, then the motion $p_0$ from $i$ to $j$ is such that $s_1\cdot p_0=i\in L$ implying $i\oplus p_0$ is a valid concatenation.

An ideal set of starts $\s$ would have the property that for every pair of vertices $i,j\in L$ with motion $p$ from $i$ to $j$, the concatenation $i\cdot p$ is valid. Obtaining such a set is simple for a class of configuration spaces described here. For motion planning for autonomous cars and car-like robots, we typically use a configuration space of the form
\begin{equation}
  \label{ConfigSpace}
  \X =\mathbb{R}^2\times[0, 2\pi)\times U_1\dots\times U_N,
\end{equation}
where $(x,y)\in \mathbb{R}^2$ represents the planar location of the vehicle, $\theta\in [0, 2\pi)$ the heading, and $u_i, i=1\dots N$ represent higher order states in a state space $u_i\in U_i\subset \mathbb{R}$. These states $u_i$ may include curvature, curvature rate, velocity, acceleration, jerk, additional angles, etc. We make the assumption that each state $U_i$ is bounded by physical constraints of the vehicle, passenger comfort, or speed limits. This allows us to write $U_i=[U_i^0, U_i^1]\subset \mathbb{R}$ for some upper and lower limits $U_i^0, U_i^1\in \mathbb{R}$. For $\X$ in \eqref{ConfigSpace}, we construct a lattice $L\subseteq \X$ by discretizing each state separately. In particular, for $\alpha, \beta \in \mathbb{R}_{>0}$, and $n_0, n_1, n_2, m_i\in \mathbb{N}$ for $i=1\dots N$, we let
\begin{equation}
\label{Lattice}
  \begin{split}
    L = &\{i\alpha, i=-n_0\dots n_0\} 
    \times \{i\beta, i=-n_1\dots n_1\}\\
    \times &\{\pi i/2^{n_2 - 1}, i=0\dots 2^{n_2} - 1\}\\
    \times &\prod_{i=1}^N \{U_i^0 + j(U_i^1 - U_i^0)/m_i\}_{j=0}^{m_i}.
  \end{split}
\end{equation}
This lattice samples $2n_0 + 1$, and $2n_1 + 1$ values of the $x, y$ coordinates, respectively, with spacing $\alpha, \beta$ between samples, respectively.  It partitions heading values $[0, 2\pi)$ and $U_i\subset[U_i^0, U_i^1]\subset\mathbb{R}, i=1\dots,N$ into $2^{n_2}$ and $m_i+1$ evenly spaced samples, respectively. For $\X, L$ from \eqref{ConfigSpace}, \eqref{Lattice}, let
\begin{equation}
\label{start set}
  \begin{split}
    \s = \Big\{(&0, 0, \theta, u_1, \dots u_N) \ : \\
        & \theta \in \{j \pi/2^{n_2 - 1}, j=0, \dots, 2^{n_2 - 2} - 1 \},\\
        & u_i \in \{U_i^0 + j(U_i^1-U_i^0)/m_i, j=0\dots m_i\} \Big\}.
  \end{split}
\end{equation}
This starting set has the property that for any vertices $i, j \in L$, the motion $p$ from $i$ to $j$ is such that $\exists s\in \s$ with $s\oplus p \in L$.

\section{\blue{Selecting a Control Set}}
\label{sec:Choosing_CS}
In this section, we \blue{describe our proposed approach for selecting a control set and} formulate the Minimum $t$-Spanning Control Set (MTSCS) problem \blue{for multi-start lattices}. We assume that for all $i,j\in \X$, there is a motion $p$ solving Problem \ref{MPP} from $i$ to $j$ in the absence of obstacles and that the cost of this motion, $c(p)$, is known, non-negative, and is equal to 0 if and only if $i=j$. Further, we assume that $c$ obeys the triangle inequality -- i.e., if $p_1$ is a motion from $i\in \X$ to $j\in \X$, and $p_2$ is a motion from $i$ to $j$ by way of $r\in \X$, then $c(p_1)\leq c(p_2)$. This cost is left general, and may represent travel time, comfort, etc.

Given $(\X, L, \s, c)$ -- configurations space, lattice, starting set and cost of motions, respectively, the idea behind multi-start control set motion planning is the following. For each $s\in \s$, pre-compute a set $\mathcal{B}_s$ of cost-minimizing feasible motions from $s$ to vertices $i\in L-\s$, and let $\mathcal{B}=\bigcup_{s\in\s}\mathcal{B}_s$. We select a subset $E\subseteq \mathcal{B}$ to use as an action set during an online search in the presence of obstacles. We offer the following definition:
\begin{definition}[Relative start] For vertex $i\in L$, the \emph{relative start} of $i$ is the vertex $R(i)\in \s$ such that for each $j\in L$, if $p_j$ is the motion from $R(i)$ to $j$ then $i\oplus p_j$ is valid.
\end{definition}
To construct a set $E\subseteq\mathcal{B}$, we choose a set $E_s$ to be the action set for lattice vertices with relative start $s$, and let $E=\bigcup_{s\in\s} E_s$. If $L$ and $\s$ are from \eqref{Lattice}, \eqref{start set}, respectively, then the relative start of any configuration $i=(x, y, \theta, u_1,\dots,u_N)\in L$ is $R(i)=(0, 0, \theta', u_1,\dots,u_N)$ where 
\begin{equation*}
  \theta' = \frac{\pi}{2^{n_2-1}} \left(\frac{2^{n_2 -1}\theta}{\pi}\mod{4}\right).
\end{equation*}
The set of available motions during an online search --- i.e., the action set --- at a configuration $i\in L$ is given by $E_{R(i)}$, and the cost of each motion $p\in E_{R(i)}$ is given by $c(p)$. 

For any subset $E\subseteq \mathcal{B}$, we denote by $\bar{E}$ the set of all tuples $(i,j)\in L^2$ such that if $p$ is the optimal feasible motion from $i$ to $j$, then $p\in E_{R(i)}$:
\begin{equation}
\label{BarE}
  \begin{split}
    \bar{E} &= \{(i,j)\in L^2 : (i\oplus p = j)\Rightarrow p\in E_{R(i)}\}.
  \end{split}
\end{equation}
Critically, solving the motion planning problem between lattice vertices using a control set $E$ is equivalent to computing a shortest path in the weighted directed graph $G=(L, \bar{E}_{\text{CF}}, c)$ where $\bar{E}_{\text{CF}}\subseteq \bar{E}$ is the set of all collision-free edges in $\bar{E}$, and for all $(i,j)\in \bar{E}_{\text{CF}}$, $c((i, j))$ is the cost of the optimal feasible motion from $i$ to $j$. This graphical representation of motion planning motivates the following definitions given the tuple $(\X, L, \s, c, E)$: 
\begin{definition}[Path using $E$]
A \emph{path using} $E$ from $s\in\s$ to $j\in L$, denoted $\pi^E(s, j)$ is the cost-minimizing path from $s$ to $j$ (ties broken arbitrarily) in the weighted, directed graph $G_{\text{Free}}=(L, \bar{E}, c)$ in the absence of obstacles . 
\end{definition}
\begin{definition}[Motion cost using $E$] The cost of a motion \emph{using} $E$ from $s\in\s$ to $j\in L$, denoted $d^E(s,j)$ is the cost of the path $\pi^E(s,j)$ in $G_{\text{Free}}=(L, \bar{E}, c)$.
\end{definition}
Let $i\in L-\s, s\in \s$ and $p$ a cost minimizing motion from $s$ to $i$. If $E=\mathcal{B}$, then  $d^E(s,j)=d^{\mathcal{B}}(s, j)=c(p)$ implying that using $\mathcal{B}$ as a control set will result in minimal cost paths. However, if $\mathcal{B}$ is large, the branching factor at a vertex $i\in L$ during an online search may be prohibitive. We therefore wish to limit the size of $|E_{R(i)}|$ while keeping $d^E(s,j)$ close to $d^{\mathcal{B}}(s, j)$ for all $j\in L$. This motivates the following definition:

\begin{definition}[$t$-Error]
Given the tuple $(\X, L, \s, c, E)$, the $t$-\emph{error} of $E$ is defined as
$$
t\text{Er}(E)=\max_{\substack{s\in\s \\ j\in L-\s}} \ \frac{d^E(s,j)}{d^{\mathcal{B}}(s,j)}.
$$
\end{definition}
That is, the $t$-error of a control set $E$ is the worst-case ratio of the distance using $E$ from a start $s$ to a vertex $j$ to the cost of the optimal path from $s$ to $j$ over all $s\in\s, j\in L$. The $t$-error can be used to evaluate the quality of a control set $E$:
\begin{definition}[$t$-Spanning Set]
Given the tuple $(\X, L, \s, c, E)$, and a real number $t\geq 1$, we say that a set $E$ is a $t$-\emph{spanner} of $L$ (or $t$-\emph{spans} $L$), if $t\emph{Er}(E)\leq t$. 
\end{definition}
Our objective is to compute a control set $E$ that optimizes a trade-off between branching factor and motion quality. This is formulated in the following problem:
\begin{problem}[Minimum $t$-spanning Control Set Problem] 
\label{MCSP}
\textbf{Input:} A tuple $(\X, L, \s, c)$, and a real number $t\geq 1$.
\\
\textbf{Output:} A control set $E=\bigcup_{s\in\s}E_s$ that $t$-spans $L$ where $\max_{s\in \s}|E_s|$ is minimized.
\end{problem}
Using a solution $E$ to Problem \ref{MCSP} as a control set has two beneficial properties. First, the number of collision-free neighbors of any vertex $i\in L$ in the graph $G$ is at most $|E_{R(i)}|$ whose maximum value is minimized. Second, for every $i\in L-\s, s \in \s$, it must hold that $d^E(s,i)\leq td^{\mathcal{B}}(s,i)$. Thus, Problem \ref{MCSP} represents a trade-off between branching factor and motion quality. As an alternative, the framework presented here can compute a $t$-spanning that minimizes $|E|$ instead of $\max_{s\in \s}|E_s|$ though this does not ensure a minimum branching factor at each vertex.  

\section{Proposed Methods}
In this section we present a MILP formulation of Problem \ref{MCSP}, and algorithms to compute and smooth motions using computed control sets.
\subsection{Multi-Start MTSCS Problem: MILP Formulation}
\label{sec:MILP}

In \cite[Theorem 5.2.1]{botros2021lattice}, we show that Problem \ref{MCSP} is NP-hard. Informally, there is no known polynomial time algorithm for solving any problem in this class (if one
existed, then it would prove P = NP). Thus, problems in NP-hard are widely believed to be intractable beyond a certain problem size. \blue{Motivated by this, we pose the problem as a MILP, an NP-hard class of problems. Several powerful MILP solvers are available, allowing many problems to be solved to optimality. These solvers are also anytime, returning feasible solutions along with sub-optimality certificates when terminated early.} Let $(\X, L, \s, c)$ denote configuration space, lattice, start set, and cost of vertex-to-vertex motions in $L$, respectively. For any motion $q\in \mathcal{B}=\bigcup_{s\in \s}\mathcal{B}_s$, let
$$
S_q=\{(i,j) \ : \ i,j\in L, \ i\oplus q =j\}.
$$ 
Thus $S_q$ is the set of all pairs $(i,j)\in L^2$ such that $q$ is a cost-minimizing feasible motion from $i$ to $j$ and $i\cdot q$ is valid. By definition of valid concatenations, there exists $s\in \s$ and $j'\in L$ such that $q$ is a cost-minimizing feasible motion from $s$ to $j'$ (i.e., $s=R(i)$) implying that $(s, j')\in S_q$.

Let $E=\bigcup_{s\in \s}E_s$ be a solution to Problem \ref{MCSP}. Let $G_{\text{Free}}=(L, \bar{E}, c)$ be the weighted directed graph with edges $\bar{E}$ given in \eqref{BarE}. For each $s\in\s$ and each $(i,j)\in \bar{E}$, we make a copy $(i,j)^s$ of the edge. This allows us to treat edges differently depending on the starting vertex of the path to which they belong. For each $r\in L$, a path using $E$ from $s$ to $r$, $\pi^E(s, r)$, may be expressed as a sequence of edges $(i,j)^s$ where $(i,j)\in \bar{E}$. For each $s\in \s$ we construct a new graph \begin{equation}
\label{Tree}
T^s=(L, E_T^s),
\end{equation}
whose edges $E_T^s$ are defined as follows: let
$$
\mathcal{P}^{s} = \bigcup_{i\in L-\s}\pi^E(s,i).
$$
Thus $\mathcal{P}^{s}$ is the set of all minimal cost paths from $s$ to vertices $i\in L-\s$ in the graph $G_{\text{Free}}$. These paths are expressed as a sequence of edge copies $(i,j)^s$ where $(i,j)\in \bar{E}$. For each $i\in L-\s$, and for each $s\in\s$, if $\mathcal{P}^{s}$ contains two paths $\pi^E_1, \pi^E_2$ from $s$ to $i$, determine the last common vertex $j$ in paths $\pi^E_1, \pi^E_2$, and delete the the copy of the edge in $\pi^E_2$ whose endpoint is $j$ from $\mathcal{P}^{s}$. The remaining edges are the set $E_T^s$. The graph $T^s$ has a useful property defined here:

\begin{definition}[Arborescence]
From Theorem 2.5 of~\cite{korte2012combinatorial}, a graph $T$ with a vertex $s$ is an arborescence rooted at $s$ if every vertex in $T$ is reachable from $s$, but deleting any edge in $T$ destroys this property.
\end{definition}
Intuitively, if $T$ is an arborescence rooted at $s$, then for each vertex $j\neq s$ in $T$, there is a unique path in $T$ from $s$ to $j$. For each $s\in \s$, the graph $T^s$ an arborescence rooted at $s$. This is shown in the following Lemma:
\begin{lemma}[Arborescence Lemma]
\label{TREELEM}
Let $E=\bigcup_{s\in \s}E_s$ be a solution to Problem \ref{MCSP}, and $G_{\text{Free}}=(L, \bar{E}, c)$. For each $s\in \s$, $T^s$ (in \eqref{Tree}) is an arborescence rooted at $s$. Further, $\forall i\in L-\s$, $d^E(s, i)$ is the length of the path in $T^s$ to $i$.
\end{lemma}
\begin{proof}
Let $s\in \s$. To show that $T^s$ is an arborescence rooted at $s$, observe that there is a path in $T^s$ from $s$ to all $i\in L$. Indeed, if $E$ solves Problem \ref{MCSP}, then $E$ $t$-spans $L$ and there must be at least one path, $\pi^E(s,i)$ using $E$ from $s$ to each $i\in L$ implying that $\pi^E(s,i)\in \mathcal{P}^{s}$. Since $E_T^s$ deletes duplicate paths from $\mathcal{P}^{s}$, there must be precisely one path in $T^s$ from $s$ to $i$ implying that $T^s$ is an arborescence rooted at $s$. The cost of the path in $T^s$ from $s$ to $i$ is defined as the cost of the path $\pi^E(s,i)$ which is $d^E(s, i)$.
\end{proof}
Lemma \ref{TREELEM} implies that $E$ is a $t$-spanner of $L$ if and only if $\forall s\in \s$ there is a corresponding arborescence $T^s$ rooted at $s$ whose vertices are $L$, whose edges $(i,j)^s$ are copies of members of $S_q$ for a $q\in E$, and where the cost of the path from $s$ to any $i\in L$ in $T^s$ is no more than a factor of $t$ from optimal. Indeed, the forward implication follows from Lemma \ref{TREELEM}, while the converse holds by definition of a $t$-spanner. From this, we develop four criteria that represent necessary and sufficient conditions for $E$ to $t$-span $L$:
\begin{description}
\item[Usable Edge Criteria:] For any $q\in \mathcal{B}$, for any $s\in \s$, and for any $(i,j)\in S_q$, the copy $(i,j)^s$ may belong to a path in $T^s$ from $s$ to a vertex $r\in L$ if and only if $q\in E_{R(i)}$.
\item[Cost Continuity Criteria:]
For any $s\in \s$, $q\in \mathcal{B}$, $(i,j)\in S_q$, and $(i,j)^s$ a copy of $(i,j)$, if $(i,j)^s$ lies in the path in $T^s$ to vertex $j$ then $c(\pi^E(s, j))=c(\pi^E(s, i)) + c(q)$. That is, the cost of the path from $s$ to $j$ in $T^s$ is equal to the cost of the path from $s$ to $i$ plus the motion from $i$ to $j$.
\item[$t$-Spanning Criteria:] For any vertex $j\in L-\s$, and $s\in\s$, the cost of the path in $T^s$ from $s$ to $j$ can be no more than $t$ times the cost of the direct motion from $s$ to $j$.
\item[Arborescence Criteria:]
The set $T^s$ must be an arborescence for all $s\in \s$.
\end{description}
\par
We now present an MILP encoding of these criteria. Let $|L|=n$ with all vertices enumerated $1,2,...,n$ with $s\in \s$ taking values $1,\dots,m$ for $m\leq n$. For any control set $E=\bigcup_{s\in\s}E_s$, define $m(n-m)$ decision variables:
$$
y_q^s = \begin{cases}
1, \ \text{if } q\in E_s\\
0, \ \text{otherwise}.
\end{cases}, \ q=m+1,\dots,n
$$
For each edge $(i,j)\in L^2$, and each $s\in\s$ let 
$$
x_{ij}^s=\begin{cases}
1 \ &\text{if } (i,j)^s\in T^s \\
0 \ &\text{otherwise}. 
\end{cases}
$$
That is, $x_{ij}^s=1$ if $(i,j)^s$ (the copy of the edge $(i,j)$ for start $s\in\s$) lies on a path from $s$ to a vertex in the lattice.
Let $z_i^s$ denote the length of the path in the tree $T^s$ to vertex $i$ for any $i\in L$, $c_{ij}$ the cost of the optimal feasible motion from $i$ to $j$, and let $L'=L-\s$. The criteria above can be encoded as the following MILP:
\begin{subequations}
\label{MILP1}
\begin{align}
\label{OF}
&\min_{K\in\mathbb{R}} \ K\\
&s.t. \ \forall s\in \s \notag \\
\label{ControsSetSize}
& \sum_{q\in L'}y_q^s \leq K, \\
\label{XY}
&x_{ij}^{s'}-y_q^s\leq 0, \  \ &&\forall(i,j)\in S_q, \\ & &&\forall q\in\mathcal{B}, \ \forall s'\in\s, \notag \\
\label{COST}
&z_i^s+c_{ij}-z_j^s\leq M_{ij}^s(1-x_{ij}^s), \ &&\forall (i,j)\in L\times L'\\
\label{COST3}
&z_j^s\leq tc_{sj}, \ &&\forall j\in L'\\
\label{TREE}
&\sum_{i\in L}x_{ij}^s=1, \ &&\forall j\in L'\\
\label{Boolx2}
&x_{ij}^s\in \{0,1\}, \ &&\forall (i,j)\in L\times L'\\
\label{Booly}
&y_q^s\in\{0,1\}, \ &&\forall q\in \mathcal{B},
\end{align}
\end{subequations}
where $M_{ij}^s=tc_{si}+c_{ij}-c_{sj}$. The objective function \eqref{OF} together with constraint \eqref{ControsSetSize} minimizes $\max_{s\in \s}|E_s|$ as in Problem \ref{MCSP}. The remainder of the constraints encode the four criteria guaranteeing that $E$ is a $t$-spanning set of $L$:

\textbf{Constraint (\ref{XY}):} Let $q$ be the motion in $\mathcal{B}$ from $s\in\s$ to $j\in L-\s$. If $q\not\in E_s$, then $y_q^s=0$ by definition. Therefore, (\ref{XY}) requires that $x_{ij}^{s'}=0$ for all $(i,j)\in S_q, s'\in \s$. Alternatively, if $y_q^s=1$, then $x_{ij}^{s'}$ is free to take values $1$ or $0$ for any $(i,j)\in S_q, s'\in \s$. Thus constraint (\ref{XY}) encodes the Usable Edge Criteria.

\textbf{Constraint (\ref{COST}):}
Constraint (\ref{COST}) takes a similar form to~\cite[Equation~(3.7a)]{desrosiers1995time}. Note that $M^s_{ij}\geq 0$ for all $(i,j)\in \mathcal{B}, s\in\s$. Indeed, $\forall t\geq 1$, $M_{ij}^s\geq c_{si}+c_{ij}-c_{sj}$, and $c_{si}+c_{ij}\geq c_{sj}$ by the triangle inequality. Replacing $M_{ij}^s$ in (\ref{COST}) yields
\begin{equation}
\label{cc}
z_i^s+c_{ij}-z_j^s\leq (tc_{si}+c_{ij}-c_{sj})(1-x_{ij}).
\end{equation}
If $x_{ij}^s=1$, then $(i,j)^s$ is on the path in $T^s$ to vertex $j$ and (\ref{cc}) reduces to $z_j^s\geq z_i^s+c_{ij}$ which encodes the Cost Continuity Criteria. If, however, $x_{ij}^s=0$, then (\ref{cc}) reduces to $z_{i}^s-z_j^s\leq tc_{si}-c_{sj}$ which holds trivially by constraint (\ref{COST3}) and by noting that $z_j^s\geq c_{sj}, \forall j\in L$ by the triangle inequality. 

\textbf{Constraint (\ref{TREE}):} Constraint (\ref{TREE}) together with constraint (\ref{COST}) yield the Arborescence Criteria. Indeed for all $s\in \s$, by Theorem 2.5 of~\cite{korte2012combinatorial}, $T^s$ is an arborescence rooted at $s$ if every vertex in $T^s$ other than $s$ has exactly one incoming edge, and $T^s$ contains no cycles. The constraint (\ref{TREE}) ensures that every vertex in $L'$, which is the set of all vertices in $T^s$ other than those in $\s$, have exactly one incoming edge, while constraint (\ref{COST}) ensures that $T^s$ has no cycles. Indeed, suppose that a cycle existed in $T^s$, and that this cycle contained vertex $i\in L'$. Suppose that this cycle is represented as
\[
i\rightarrow j\rightarrow \dots\rightarrow k\rightarrow i.
\]
Recall that it is assumed that the cost of any motion between two different configurations in $\X$ is strictly positive. Therefore, (\ref{COST}) implies that $z_i^s<z_j^s$ for any $(i,j)\in L^2$. Therefore, 
\[
z_i<z_j<\dots<z_k<z_i,
\]
which is a contradiction.

 \begin{figure*}[t]
\centering
  \includegraphics[width = 0.9\linewidth]{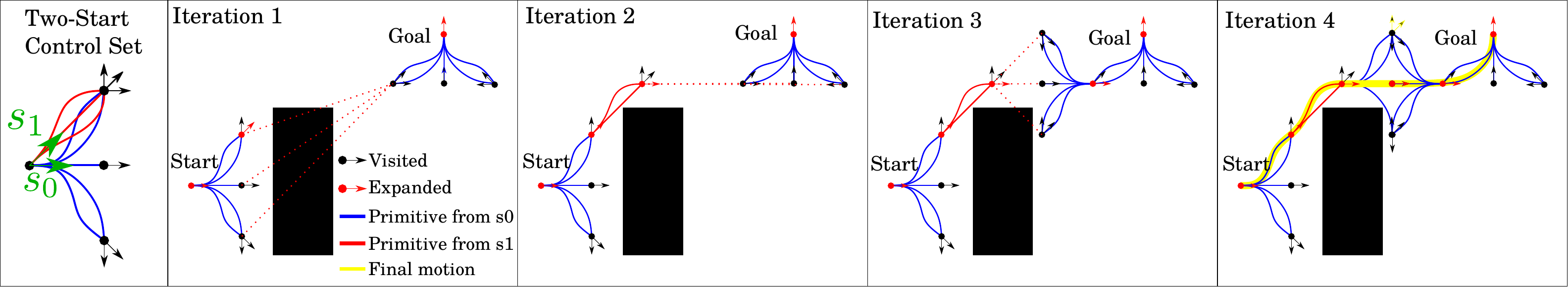}
  \caption{Example motion planning using PrAC for a 2-start lattice.}
  \label{UsingPrAC}
\end{figure*}
 
\subsection{Motion Planning With a MTSCS}
\label{sec:AstarVar}
 The previous section illustrates how to compute a control set $E=\bigcup_{s\in\s}E_s$ given the tuple $(\X, L, \s, c)$. We have also presented a description of how such a control set may be used during an online A*-style search of the graph $G=(L, \bar{E}_{\text{CF}}, c)$ from a start vertex $p_s\in L$ to a goal vertex $p_g\in L$. Here, the edge set $\bar{E}_{\text{CF}}$ is the set of pairs $(i,j)\in L^s$ such that there exists a collision-free motion $p\in E_{R(i)}$ from $i$ to $j$ where $R(i)\in \s$ is the relative start of $i$. The motion primitives in this control set could be used in any graph-search algorithm. In this section, we propose one such algorithm: an A* variant, \emph{Primitive A* Connect} (PrAC) for path computation in $G$. The algorithm follows the standard A* algorithm closely with some variations described here:
 
\textbf{Weighted fScore:} In the standard A* algorithm, two costs are maintained for each vertex $i$: the \emph{gScore}, $g(i)$ representing the current minimal cost to reach vertex $i$ from the starting vertex $p_s$, and the $\emph{fScore}$, $f(i) = g(i) + h(i)$ which is the sum of the gScore and an estimated cost to reach the goal vertex $p_g$ by passing through $i$ given heuristic $h$. We implement the weighted fScore from \cite{pohl1970heuristic}: Given a value $\lambda\in[0, 1]$, we let $a=0.5\lambda, b= 1-0.5\lambda$, and we define a new cost function $f'=ag(i) + bh(i)$. If $\lambda = 1$, then both gScore and heuristic are weighted equally and standard A* is recovered. However, as $\lambda$ approaches 0, more weight is placed on the heuristic which promotes \emph{exploration} over \emph{optimization}. While using a value $\lambda < 1$ eliminates optimality guarantees, it also empirically improves runtime performance. In practice, using $\lambda=1$ works well for maneuvers where $p_s$ and $p_g$ are close together, like parallel parking, while small values of $\lambda$ work well for longer maneuvers like traversing a parking lot. 

 \textbf{Expanding Start and Goal Vertices:} We employ the bi-directional algorithm from \cite{chen2017front} with the addition of a direct motion. In detail, we expand vertices neighboring both start and goal vertices and attempt to connect these vertices on each iteration. In essence, we double the expansion routine at each iteration of A* once from $p_s$ to $p_g$ and once from $p_g$ to $p_s$ with reverse orientation. We maintain two trees, one rooted at the start vertex $p_s$, denoted $T_s$, the other at the goal, $T_g$ whose leaves represent open sets $O_s, O_g$, respectively. We also maintain two sets containing the current best costs $g(p_s, i)$ to get from $p_s$ to each vertex $i\in T_s$, and from $p_g$ to each vertex $j\in T_g$ (traversed in reverse), $g(p_g, j)$, respectively. Given an admissible heuristic $h'$, we define a new heuristic $h$ as
  \begin{equation}
 \label{NewHeuristic}
 \begin{split}
   h(i) = \min_{j \in O_{g}} h'(i, j)+g(p_g, j), \ \forall i\in O_{s}\\
   h(j) = \min_{i \in O_s}h'(j, i) + g(p_s, i), \ \forall j\in O_g
 \end{split}
 \end{equation}
On each iteration of the A* while loop, let $i\in O_s$ be the vertex that is to be expanded, let $j\in O_g$ be the vertex that solves \eqref{NewHeuristic} for $i$, and let $p$ denote the pre-computed cost-minimizing feasible motion from $i$ to $j$. We expand vertex $i$ by applying available motions in $E_{R(i)}\cup\{p\}$ where $E_{R(i)}$ is the set of available motions at the relative start of $i$. The addition of $p$ to the available motions improves the performance of the algorithm by allowing quick connections between $T_s, T_g$ where possible. In the same iteration, we then swap $T_s$ and $T_g$ and perform the same steps but with all available motions in reverse. This is illustrated in Figure \ref{UsingPrAC}. Expanding start and goal vertices has proven especially useful during complex maneuvers like backing into a parking space.

\textbf{Off-lattice Start and Goal Vertices:} For a configuration $v\in \X$ and a lattice $L$ given by \eqref{Lattice}, $\text{Round}(v)$ denotes a function that returns the element of $L$ found by rounding each state of $v$ to the closest value of that state in $L$. For lattice $L$ with start set $S$ given by \eqref{start set} and control set $E$ obtained by solving the MILP in \eqref{MILP1}, it is likely that the start and goal configurations $p_s, p_g\in \X$ do not lie in $L$. This is particularly true in problems that necessitate frequent re-planning. While increasing the fidelity of the lattice or computing motions online between lattice vertices and $p_s, p_g$ (e.g., \cite{oliveira2018combining}) can address this issue, both these methods adversely affect the performance of the planner. Instead we propose a method using lattices with graduated fidelity, a concept introduced in \cite{pivtoraiko2009differentially}. We compute a control set $E_{\text{off}}$ of primitives from off-lattice configurations to lattice vertices. These can be traversed in reverse to bring lattice vertices to off-lattice configurations. The technique is summarized in Algorithm \ref{OffLattice} which is executed offline. The algorithm takes as input $L, E$, and a vector of tolerances for each state $Tol=(x_{\text{tol}}, y_{\text{tol}}, \theta_{\text{tol}}, u_{1, \text{tol}},\dots,u_{N,\text{tol}})$. It first computes a set $Q\subset\X$ such that for every configuration in $v\in\X$ there exists a configuration $q\in Q$ that can be translated to $q'\in \X$ where each state of $q'$ is within the accepted tolerance of $v$ (Line 2). For each element of $Q$, we determine the lattice vertex $q'=\text{Round}(q)$, and the set of available actions at the relative start of $q'$, $E_{R(q')}$ . For each primitive in $E_{R(q')}$, we compute a motion from $q$ to a lattice vertex close to the endpoint of $p$ and store it in a set $E_q$ (Lines 7-13). In Lines 9, 10 the vertex $p$ is modified to $p'$ that is no closer to $q$ than $p$. This is to ensure that the motion added to $E_q$ in Line 13 does not posses large loops not present the primitive $p$. These loops can arise if the start and end configurations of a motion are too close together. An illustrative example of this principle is given in Figure \ref{OffLatticeFig}. 

\begin{figure}[h]
\centering
  \includegraphics[width = 0.75\linewidth]{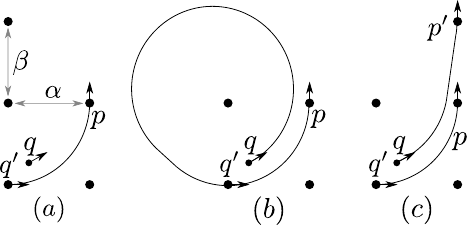}
  \caption{Algorithm \ref{OffLattice} Lines 9-13 Example. (a) Configuration $q\in \X-L$, lattice vertex $q'=\text{Round(q)}$, and primitive $p\in E_{R(q')}$. (b) Motion from $q$ to end of $p$ results in a loop as $q, p$ are too close together. (c) Vertex $p$ replaced with neighboring vertex $p'$ and a motion from $q$ to $p'$ is computed.}
  \label{OffLatticeFig}
\end{figure}

\begin{theorem}[Completeness]
If $\lambda=1$ and $p_s, p_g\in L$, then PrAC returns the cost-minimizing path in $G=(L, \bar{E}_{\text{CF}}, c)$.
\end{theorem}
\begin{proof}
Given the assumptions of the Theorem, we may view PrAC as standard A* where the expansion of $T_g$ serves only to update the heuristic. That is, PrAC reduces to standard A* with a heuristic given by \eqref{NewHeuristic} for all $i\in O_s$. At each iteration of A*, $h$ is improved by expanding $T_g$. By the completeness of A* for admissible heuristics and finite branching factor (which is the case here for finite control sets), it suffices to show that $h$ is indeed admissible. At each iteration of A*, let $i$ be the vertex in $O_s$ that is to be expanded. Let $\pi$ be the optimal path from $i$ to $p_g$ in $G$. Then $\pi$ must pass through a vertex $r$ in $O_g$. This result holds by the completeness of A* with initial configuration $p_g$ with paths traversed in reverse. Then by the triangle inequality, $c(\pi)\geq c(\pi^{\bar{E}_{\text{CF}}}(i, r))+g(p_g, r)$. Finally, because $h'$ is an admissible heuristic, $c(\pi^{\bar{E}_{\text{CF}}}(i, r))+g(p_g, r)\geq h'(i, r)+g(p_g, r)\geq \min_{j\in O_g}h'(i, j)+g(p_g, j)=h(i)$. Therefore, $h(i)\leq c(\pi)$ which concludes the proof.
\end{proof}

\begin{algorithm}
\caption{Generate Off-Lattice Control Set}
\label{OffLattice}
\small
\begin{algorithmic}[1]
\Procedure{OffLattice}{$L, E, Tol$}
\State $E_{\text{off}}\gets \emptyset$
\small
\blue{
\begin{equation*}
    \begin{split}
         &[U_x^0, U_x^1], [U_y^0, U_y^1], [U_{\theta}^0, U_{\theta}^1] \gets [0, \alpha], [0, \beta], [0, 2\pi - 2\theta_{\text{tol}}]\\
        &N_p\gets \lceil \nicefrac{(U_p^1-U_p^0)}{2u_{p_{\text{tol}}}} \rceil, p=x, y, \theta,1,\dots N\\
        &Q \gets \prod_{p\in\{x, y, \theta, 1,\dots, N\}} \{U_p^0 + \nicefrac{(U_p^1-U_p^0)j}{N_{p}}\}_{j=0}^{N_{p}}
    \end{split}
\end{equation*}
}
\normalsize

\For{$q\in Q$}
\State $(x_q, y_q, \theta_q, u_{1,q},\dots, u_{N,q})\gets q$
\State $E_{q}\gets \emptyset$
\State $q'\gets \text{Round}(q)$
\For{$p\in E_{R(q')}$}
\State $(x_p, y_p, \theta_p, u_{1,p},\dots, u_{N,p})\gets p$
\For{$x\in\{x_p, x_p+\text{sign}(x_p-x_q)\alpha\}$}
\For{$y\in \{y_p, y_p + \text{sign}(y_p-y_q)\beta\}$}
\State $p'\gets (x,y, \theta_p,  u_{1,p},\dots, u_{N,p})$
\State Compute motion from $q$ to \blue{$p'$}
\EndFor
\EndFor
\State Add motion of lowest cost over all $p'$ to $E_{q}$
\EndFor
\State Add $E_{q}$ to $E_{\text{off}}$
\EndFor
\State \textbf{return} $E_{\text{off}}$
\EndProcedure
\end{algorithmic}
\normalsize
\end{algorithm}

\subsection{Motion Smoothing}
\label{sec:Smoothing}
Given the tuple $(\X, L, \s, c, E)$, of configuration space, lattice, cost function, and MTSCS, respectively, let $G=(L, \bar{E}_{\text{CF}}, c)$ be the weighted directed graph with edge set $\bar{E}_{\text{CF}}$ is the collision-free subset of $\bar{E}$ given in \eqref{BarE}.

We now present a smoothing algorithm based on the \emph{shortcut} approach that takes as input a path in $G$, here $\pi^E(p_s, p_g)$, between start and goal configurations $p_s, p_g$. This path is expressed as a sequence of edges in $\bar{E}_{\text{CF}}$. Thus, $\pi^E(p_s, p_g)=\{(i_r, i_{r+1}), r=1, \dots, m-1\}$ for some $m\in\mathbb{N}_{\geq 2}$ where $(i_r, i_{r+1})\in \bar{E}_{\text{CF}}$ for all $r=1,\dots, m-1$ and $i_1=p_s, i_m = p_g$. Let $C_{\pi}$ denote the set of all configurations along motions from $i_r$ to $i_{r+1}$ for all $(i_r, i_{r+1})\in \pi^E(p_s, p_g)$. Algorithm \ref{ALGO2} summarizes the proposed approach.

\begin{algorithm}
\small
\caption{Smoothing Lattice Motion}
\label{ALGO2}
\begin{algorithmic}[1]
\Procedure{DagSmooth}{$\pi^E(p_s, p_g), C_{\pi}, \ChiObs, c, n, \Psi$}
\State$V_1\gets$SampleRandom$(n, C_{\pi})$
\State $V \gets \{i_r\}_{r=1}^m\cup V_1$ in the order that they appear in $C_{\pi}$
\State $\{i_r\}_{r=1}^{m + n}\gets V$
\For{$i\in V$}
\State $\text{dist}(i)=\infty$
\EndFor
\State $\text{dist}(i_1)=0$
\State $\text{Pred}(i_1)=$None
\For{$u$ from 1 to $n + m-1$}
\For{$v$ from $u + 1$ to $m + n$}
\State $p_1 = $ motion from $i_u$ to $i_v$
\State $p_2 = p_1$ motion from $i_v$ to $i_u$
\State $\underline{c}=\min\big(c(p_1), \Psi(c(p_2))\big)$, \State$\underline{p}=\text{arg}\min\big(c(p_1), \Psi(c(p_2))\big)$
\State $\overline{c}=\max\big(c(p_1), \Psi(c(p_2))\big)$, \State$\overline{p}=\text{arg}\max\big(c(p_1), \Psi(c(p_2))\big)$
\If{$\text{dist}(i_u) + \underline{c}\leq \text{dist}(i_v)$}
\If{CollisionFree$(\underline{p}, \ChiObs)$}
\State $\text{dist}(i_v) = \text{dist}(i_u) + \underline{c}$
\State $\text{Pred}(i_v)=i_u$
\Else
\If{$\text{dist}(i_u) + \overline{c}\leq \text{dist}(i_v)$}
\If{CollisionFree$(\overline{p}, \ChiObs)$}
\State $\text{dist}(i_v) = \text{dist}(i_u) + \overline{c}$
\State $\text{Pred}(i_v)=i_u$
\EndIf
\EndIf
\EndIf
\EndIf
\EndFor
\EndFor
\textbf{return} Backwards chain of predecessors from $i_m$
\EndProcedure
\end{algorithmic}
\normalsize
\end{algorithm}

Algorithm \ref{ALGO2} takes as input a collision-free, dynamically feasible path $\pi^E$ between start and goal configurations -- such as that returned by PrAC -- as well as the set of all configurations $C_{\pi}$, obstacles $\ChiObs$, cost function of motions $c$, and a non-negative natural number $n$. It also takes a function $\Psi:\mathbb{R}\rightarrow \mathbb{R}$ which is used to penalize reverse motion. That is, given two configurations $i,j\in \X$ with a motion $p$ from $i$ to $j$, we say that the cost of the motion $p$ is $c(p)$, while the cost of the reverse motion $p'$ from $j$ to $i$ that is identical to $p$ but traversed backwards is $c(p')=\Psi(c(p))$. 

The set $V$ in Line 3 represents sampled configurations along the set of configurations $C_{\pi}$ which includes all endpoints of the edges $(i_r, i_{r+1})\in \pi^E$, as well as $n$ additional random configurations along motions connecting these endpoints. Configurations in $V$ must be in the order in which they appear along $\pi^E$. For each pair ($i_u, i_v$) of configurations with $i_v$ appearing after $i_u$ in $\pi^E$, Algorithm \ref{ALGO2} attempts to connect $i_u$ to $i_v$ with either a motion from $i_u$ to $i_v$ (Line 11) or from $i_v$ to $i_u$ traversed backwards, selecting the cheaper of the two if possible (Lines 17-25).

Were we to form a graph with vertices $V$ and with edges $(i_u, i_v)\in V^2$ where $i_v$ occurs farther along $C_{\pi}$ than $i_u$ and where the optimal feasible motion from $i_u$ to $i_v$ is collision-free, then observe that this graph would be a directed, acyclic graph (DAG). Indeed, were this graph not acyclic, then the path $\pi^E(p_s, p_g)$ would contain a cycle implying that a configuration would appear at least twice in $\pi^E(p_s, p_g)$. This is impossible by construction of the PrAC algorithm which maintains two trees rooted at $p_s, p_g$, respectively. This motivates the following observation and Theorem:
\begin{observation}
\label{obs:DAG}
The nested for loop in lines 9-10 of Algorithm \ref{ALGO2} constructs a DAG and finds the minimum-cost path from $p_s$ to $p_g$ in the graph. 
\end{observation}
\begin{theorem}
Let $\pi_1^E$ be the input path to Algorithm \ref{ALGO2} between configurations $p_s$, $p_g$ with cost $c(\pi^E_1)$. Let $\pi^E_2$ be the path returned by Algorithm \ref{ALGO2} with cost $c(\pi^E_2)$. Then $c(\pi^E_2)\leq c(\pi^E_1)$. Moreover, if $n=0$, this algorithm runs in time quadratic in the number of vertices along $\pi^E_1$. 
\end{theorem}
\begin{proof}
By Observation \ref{obs:DAG}, Algorithm \ref{ALGO2} constructs a DAG containing configurations $p_s, p_g$ as vertices. It also solves the minimum cost path problem on this DAG. Observe further that by construction of the DAG vertex set in Line 3, all configurations along the original path $\pi^E_1$ are in $V$. Thus $\pi^E_1$ is an available solution to the minimum path problem on the DAG. This proves that the minimum cost path can do no worse than $c(\pi^E_1)$. If $n=0$, then $V$ is the set of endpoints of edges in $\pi^E_1$. Therefore, $V\subseteq L$. Because all motions between lattice vertices have been pre-computed in $\mathcal{B}$, Lines 11,12 can be executed in constant time, and the nested for loop in lines 9-10 will thus run in time $O(m)^2$ where $m$ is the number of vertices on the path $\pi^E_1$.
\end{proof}
\begin{figure}[h]
\centering
  \includegraphics[width = 0.8\linewidth]{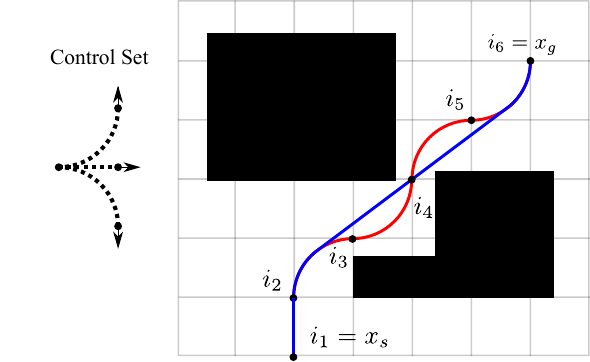}
  \caption{Comparison of smoothing algorithms for same input path (red). Blue: path smoothed using Algorithm \ref{ALGO2}. Red: path smoothed using Algorithm 1 from \cite{oliveira2018combining} (no change from input path).}
  \label{CompareSmooth}
\end{figure}

Algorithm \ref{ALGO2} is similar to Algorithm 1 in \cite{oliveira2018combining} with one critical difference. The latter algorithm adopts a greedy approach, connecting the start vertex $i_1$ in the path $\pi^E$ to the farthest vertex $i_k\in \pi^E$ that can be reached without collisions. This process is then repeated at vertex $i_k$ until the goal vertex is reached. This greedy approach also runs in time quadratic in the number of vertices along the input path but the cost of paths returned will be no lower than the cost of paths returned by Algorithm \ref{ALGO2} proposed here. A simple example is illustrated in Figure \ref{CompareSmooth}. Using the control set in Figure \ref{CompareSmooth}, an initial path is computed from $i_1=p_s$ to $i_6=p_g$ (red). This path remains unchanged when smoothed using the proposed Algorithm in \cite{oliveira2018combining}. However, Algorithm \ref{ALGO2} will return the less costly blue path.

\section{Experimental Results}
\label{ResultsSection}
We verify our proposed technique against two techniques in two common navigational settings: Hybrid A* \cite{dolgov2008practical}, and the hybrid motion primitive, numerical optimization approach (denoted CGPrim) from \cite{zhang2018hybrid}. We consider two variants of Hybrid A*: an unsmoothed version (HA*) and a smoothed version (S-HA*) that uses the smoothing approach outlined in \cite{dolgov2008practical}. The techniques described here were implemented in Python 3.7 (Spyder). Results were obtained using a desktop equipped with an AMD Ryzen 3 2200G processor and 8GB of RAM running Windows 10 OS. Start and goal configurations were not constrained to be lattice vertices. We assume that obstacles are known to the planner ahead of time and that the environment is noiseless. As such, the results that follow can be thought of as a single iteration of a full re-planning process.

\subsection{Memory}
The motions used in this section are $G3$ curves which take the form CCC, or CSC (``C'' denotes a curved segment, and ``S'' a straight segment). Such motions can be easily generated from 9 constants in the case of motion planning without velocity, and 19 constants otherwise \cite{botros2022tunable}. Instead of storing all configurations along a motion, we store only these constants. 

\subsection{Parking Lot Navigation}
We begin by validating the proposed method against HA* and S-HA* in a parking lot scenario. Though HA* is not a new algorithm, more recent state of the art approaches use HA* to plan an initial motion which is then refined (e.g.~\cite{9062306, zhang2018autonomous}). We are therefore motivated to compare the run-time and path quality of the approach proposed in this paper to HA* whose run-time is a lower bound of all state of the art algorithms using it as a sub-routine. In \cite{botros2022tunable}, we present a method of computing motions between start and goal configurations in the configurations space $\X = \mathbb{R}^2\times [0, 2\pi) \times \mathbb{R}^3$. Here, configurations take the form $(x,y,\theta, \kappa, \sigma, \rho)$ where $(x,y, \theta)$ represent the planar coordinates and heading of a vehicle, $\kappa$ the curvature, $\sigma$ the curvature rate (defined as $d\kappa/ds$ for arc-length $s$), and $\rho$ the second derivative of curvature with respect to arc-length. We generate motions with continuously differentiable curvature profiles assuming that $\kappa, \sigma, \rho$ are bounded in magnitude. The motivation for this work comes from the observation that jerk, the derivative of acceleration with respect to time, is a known source of discomfort for the passengers of a car \cite{levinson2011towards}. In particular, minimizing the integral of the squared jerk is often used a cost function in autonomous driving \cite{zhang2018hybrid}. Since this value varies with $\sigma$, keeping $\sigma$ low and bounded is desirable in motion planning for autonomous vehicles. 

Unfortunately, due to the increased complexity of the configuration space over, say, the configuration space used in the development of Dubins' paths, $\X_{\text{Dubins}}=\mathbb{R}^2\times [0, 2\pi)$, solving TPBV problems in $\X$ takes on average two orders of magnitude more time than computing a Dubins' path. Though motions computed using the techniques in \cite{banzhaf2018g}, \cite{botros2022tunable} are more comfortable, and result in lower tracking error than Dubins' paths, they may be computationally impractical to use in motion planners that require solving many TPBV problems online.  However, they prove to be particularly useful in the development of pre-computed motion primitives. 

\subsubsection{Lattice Setup \& Pruning}
The configuration space used here is $\X=\mathbb{R}^2\times [0, 2\pi)\times \mathbb{R}^3$ with configurations $(x,y, \theta, \kappa, \sigma, \rho)$. Motion primitives were generated using the MILP in \eqref{MILP1} for a $15\times 20$ square lattice with 16 headings and 3 curvatures, and a value of $t=1.1$ (10\% error from optimal). To account for the off-lattice start-goal pairs, we used a higher-fidelity lattice with 64 headings and 30 curvatures. Lattice vertex values of $\sigma, \rho$ were set to 0. This results in a start set $\s$ with 12 starts given by \eqref{start set}. The cost $c$ of a motion is defined by the arc-length of that motion. These motions were computed using our work in \cite{botros2022tunable} with bounds on $\kappa, \sigma, \rho$:
\begin{equation}
\label{realbounds}
  \kappa_{\text{max}} = 0.1982 m^{-1}, \ \sigma_{\text{max}} = 0.1868 m^{-2} \ 
  \rho_{\text{max}} = 0.3905 m^{-3},
\end{equation}
which were deemed comfortable for a user \cite{banzhaf2018g}, particularly at low speeds typical of parking lots. The spacing of the lattice $x,y$-values was chosen to be $r_{\text{min}}/4$ for a minimum turning radius $r_{\text{min}}=1/\kappa_{\text{max}}$. Finally, if the cost of the motion from $s\in\s$ to a vertex $j$ was larger than 1.2 times the Euclidean distance from $s$ to $j$ for all $s\in \s$, then $j$ was removed from the lattice. This technique which we dub \emph{lattice pruning} is to keep the lattice relatively small, and to remove vertices for which the optimal motion requires a large loop. The value of 1.2 comes from the observation that the optimal motion from the start vertex $s=(0, 0, 0, \kappa_{\text{max}})$ to $j=(r_{\text{min}}, r_{\text{min}}, \pi/2, \kappa_{\text{max}})$ is a quarter circular arc of radius $r_{\text{min}}$. The ratio of the arc-length of this maneuver to the Euclidean distance from $s$ to $j$ is $\pi/(2\sqrt{2})\approx 1.11$. Thus using a cutoff value of 1.2 admits a sharp left and right quarter turn but is still relatively small.

\subsubsection{Adding Reverse Motion} The motion primitives returned by the MILP in \eqref{MILP1} are motions between a starting vertex $s\in \s$, and a lattice vertex $j\in L - \s$. As such, they are for forward motion only. To add reverse motion primitives to the control set $E_s$ for each $s\in\s$, we applied the forward primitives to $\s$ in reverse. We then rounded the final configurations of these primitives to the closest lattice vertex. For each $(x,y)$-value of the final configurations, we select a single configuration $(x,y,\theta, \kappa)$ that minimizes arc-length. This is to keep the branching factor of an online search low. Finally, to each $s=(0, 0, \theta, \kappa)\in \s$ we add three primitives: $(0, 0, \theta, \pm \kappa_{\text{max}}), (0, 0, \theta, 0)$ with a reverse motion penalty. These primitives reflect the cars ability to stop and instantaneously change its curvature. 

\subsubsection{Scenario Results}We verified our results in five parking lot scenarios (a)-(e). The first four scenarios illustrate our technique in parking lots requiring forward and reverse parking. The results are illustrated in Figure \ref{ParkingLotResults}. Here, we have compared our approach to S-HA* using an identical collision checking algorithm, and using the same heuristic (that proposed in \cite{dolgov2008practical}). Initial paths for our approach were computed using $\lambda =0.2$. These paths were then smoothed via Algorithm \ref{ALGO2} with $n=0$. Though the motions may appear similar, they are actually quite different. 
\begin{figure}[t]
\centering
  \includegraphics[width = 0.9\linewidth]{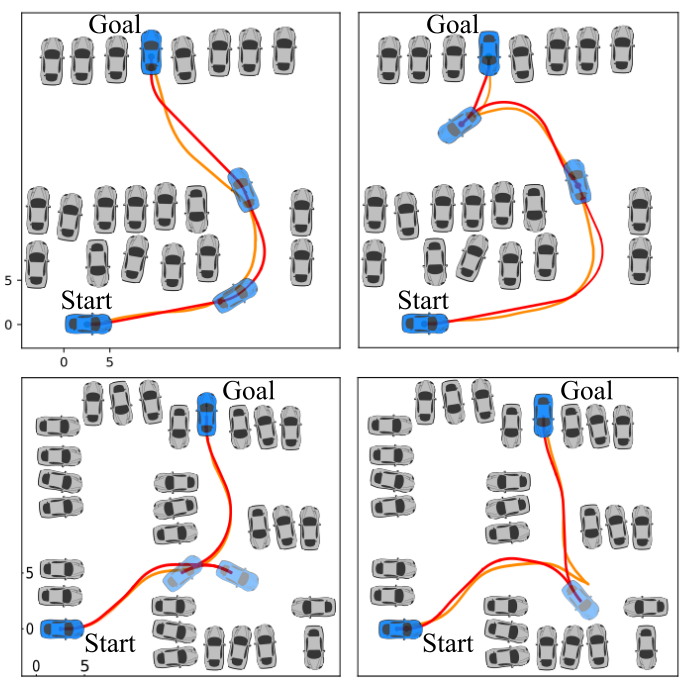}
  \caption{Scenarios (a) - (d) (top left to bottom right). Red paths from proposed method, yellow from S-HA*.}
  \label{ParkingLotResults}
\end{figure}
To evaluate the quality of the motions predicted, we use three metrics: the integral of the square jerk (IS Jerk), final arc-length, and runtime, and the results are summarized in Table \ref{PathSave}. These three metrics are expressed as ratios of the value obtained using S-HA* to those of the proposed. Values for the proposed motions before smoothing (using Algorithm \ref{ALGO2}) are presented in parenthesis. The runtime values for S-HA* for scenarios (a)-(e) were 280ms, 370ms, 275ms, 390ms, and 410ms, respectively. Because these techniques are determinisitc, no standard deviations are presented. The major difference between the two approaches can be seen in the IS Jerk Reduction. Because the motion primitives we employ are each $G3$ curves with curvature rates bounded by what is known to be comfortable the resulting IS Jerk values of proposed paths are up to 16 less than those computed with S-HA*. In fact, using the S-HA* approach may result in motions with infeasibly large curvature rates resulting in larger tracking errors and increased danger to pedestrians. 

Despite the bounds on curvature rate, the final arc-lengths of curves computed using our approach are comparable to those of S-HA*. Further, though Reeds-Shepp paths (which are employed by HA* before smoothing) take on average two orders of magnitude less time to compute than $G3$ curves, the runtime performance of our method often exceeds that of S-HA*. In fact, our proposed method takes, on average 6.9 times less time to return a path, exceeding the average runtime speedup of the method proposed in \cite{zhang2018hybrid}. Moreover, the methods in \cite{zhang2018hybrid} do not account for reverse motion, and assume that a set of way-points between start and goal configurations is known. It should be noted that values for HA* were not included in Table \ref{PathSave} owing to the moments of infinite jerk experienced at transitions between curvatures. However, the proposed approach took on average 4.6 times less time to compute than paths using HA* and featured an average length reduction of 1.4 over HA*.

The only scenario in which S-HA* produces a motion in less time than the proposed method is Scenario (c) in which S-HA* produced a path with no reverse motion (which accounts for the speedup). However, in order to produce this motion the path must feature moments of very large changing curvature resulting in an IS Jerk that is 16.3 times higher than the proposed method. The low run-time of the proposed method may due to the length of primitives we employ. It has been observed that HA* (which is used as a initial path for S-HA*) often takes several iterations to obtain a motion of comparable length to one of our primitives. This results in a much larger open set during each iteration of A*.

The final scenario we investigated is a parallel parking scenario (scenario (e)) illustrated in Figure \ref{ParallelParking}. Though the S-HA* motion appears simpler, it requires an IS Jerk 16.7 times larger than what is considered comfortable. It should be noted that several other parallel parking scenarios in which the clearance between obstacles was decreased were considered. While the proposed method returned a path in each of these scenarios, HA* (and thus S-HA*) failed to produce a path in the allotted time. A further complex parking lot navigation scenario using the proposed approach is shown in Figure \ref{TspanEx}.  

The resolution of the lattice used was $\sim 1.26m$. To verify the efficacy of our approach with higher resolution, we repeated the experiments above with a resolution of $0.5m$ as in \cite{dolgov2008practical} with $t=1.2, \lambda=0.1$. Though the size of the lattice increased by a factor of 6.4, the control set only increased by a factor of 4.1. Runtimes for these experiments was on average 1.8 times faster than S-HA* and 1.2 times faster than HA*. IS Jerk and length reduction increased by less than 1\% after smoothing (via Algorithm \ref{ALGO2}) as compared to the proposed approach with lower resolution.

\begin{figure}[t]
\centering
  \includegraphics[width = 0.77\linewidth]{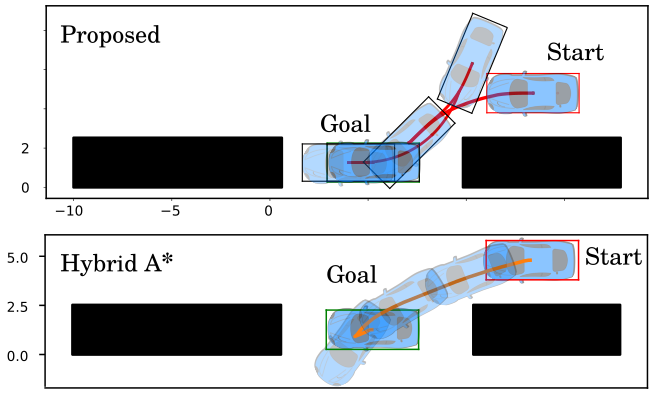}
  \caption{Scenario (e).}
  \label{ParallelParking}
\end{figure}

\begin{table}
\centering
\begin{tabulary}{\textwidth}{C C C C} 
\toprule
 Scenario & IS Jerk Reduction & Length Reduction & Runtime Speedup \\
\midrule
$(a)$ & 7.7 (7.3) & 1.0 (0.8) & 5.3 (5.9)\\
$(b)$ & 9.9 (9.5) & 1.0 (0.8) & 18.4 (18.7)\\
$(c)$ &16.3 (15.6) & 0.8 (0.7) & 0.6 (0.8)\\
$(d)$ & 10.1 (9.9) & 0.9 (0. 8) & 7.3 (7.6)\\
$(e)$ & 16.7 (13.0) & 0.6 (0.2) & 3.0 (3.8)\\
\bottomrule
\end{tabulary}
\caption{Scenario Results}
\label{PathSave}
\end{table}

\subsection{Speed Lattice}
In this experiment, we generate a full trajectory (including both path and speed profile) for use in highway driving using our approach. Here, we use only forward motion as reverse motion on a highway is unlikely. This section compares the proposed method to HA*, S-HA*, and CGPrim. This latter approach was selected as a basis of comparison because it is a state-of-the-art approach that combines the use of motion primitives and numerical optimization.

In addition to developing $G3$ paths, our work in \cite{botros2022tunable} details a method with which a trajectory with configurations $(x, y, \theta, \kappa, \sigma, \rho, v, a, \beta)$ may be computed. Here, $v, a, \beta$ denote velocity and longitudinal acceleration, and longitudinal jerk respectively. The approach is to compute profiles of $\rho$ and $\beta$ that result in a trajectory that minimizes a weighted sum of undesirable trajectory features including the integral of the square (IS) acceleration, IS jerk, IS curvature, and final arc-length. The key feature of this approach is that both path (tuned by $\rho$) and velocity profile (tuned by $\beta$) are optimized simultaneously, keeping path planning in-loop during the optimization. 

Computing trajectories via the methods outlined in \cite{botros2022tunable} requires orders of magnitude more time than simple Reeds-Shepp paths. However, pre-computing a set of motion primitives where each motion is computed using the methods of \cite{botros2022tunable} ensures that every motion used in PrAC is optimal for the user. Moreover, because we include velocity in our configurations (and primitives) we do not need to compute a velocity profile.

\subsubsection{Lattice Setup \& Pruning} Motion primitives were generated \eqref{MILP1} for a $24\times 32$ grid. Dynamic bounds for comfort were kept at \eqref{realbounds}. The $x$ component of the lattice vertices were sampled every $r_{\text{min}}/6$ meters while the $y$ components were sampled every $r_{\text{min}}/12$. Headings were sampled every $\pi/16$ radians (32 samples). We also assumed values of $\kappa=\sigma=\alpha=0$ on lattice vertices. To account for off-lattice start-goal pairs, we use a higher-fidelity lattice with 128 headings, and 10 curvatures between $-\kappa_{\text{max}}, \kappa_{\text{max}}$. Finally, five evenly spaced velocities were sampled between 15 and 20 km/hr. It should be noted that this range could be changed to include 0 if that is desired without altering the methodology proposed here. Further, the performance of the proposed approach is similar for differing ranges if the fineness of the discretization is unchanged. A value of $n=0$ was used for Algorithm \ref{ALGO2}.

\subsubsection{Scenario Results} The highway scenario was chosen to closely resemble the roadway driving scenario in \cite{zhang2018hybrid}. Initial paths for our approach were computed using $\lambda =0.9$ and were smoothed via Algorithm \ref{ALGO2} with $n=0$ (though minimal smoothing was required). Results of this scenario are in Figure \ref{CompleteTraj} while performance analysis is summarized in Table \ref{PathSave2}. Performance is measured with three metrics: arc-length of the proposed motion, smoothness cost of the motion, maximum curvature obtained over the motion, the runtime speedup relative to HA*. The final column of the Table indicates whether a velocity profile was included during the motion computation. In Table \ref{PathSave2}, two values of smoothness cost are given in the form of a tuple $(\text{Smoothness}_1, \text{Smoothness}_2)$ where
\begin{equation*}
  \begin{split}
    \text{Smoothness}_1 &= \sum_{i=1}^{N-1} |\Delta\mathbf{x}_{i + 1} - \Delta\mathbf{x}_{i} |^2,\\
    \text{Smoothness}_2 &= \int_{0}^{s_f}\kappa(s)^2 ds.
  \end{split}
\end{equation*}
The first measure is used in \cite{dolgov2008practical}, where $N$ configurations are sampled along a motion, with $\textbf{x}_i$ the vector of $x,y$-components of the $i^{th}$ configuration, and $\Delta\textbf{x}_i=\textbf{x}_i-\textbf{x}_{i-1}$. The second measure is used in \cite{zhang2018hybrid}. The first three methods appearing in the Table \ref{PathSave2} are computed directly from the motions in Figure \ref{CompleteTraj}, while the last comes from \cite{zhang2018hybrid} for an identical experiment. 

The authors of \cite{zhang2018hybrid} report an average runtime speedup of 4.5 times as compared to HA* for the path planning phase (without speed profile). On average, the proposed computed a full motion, including a speed profile 4.7 times faster than the time required for HA* to compute a path. Further, the proposed approach features significantly reduced the smoothness costs.
\begin{figure}[h]
\centering
  \includegraphics[width = 0.93\linewidth]{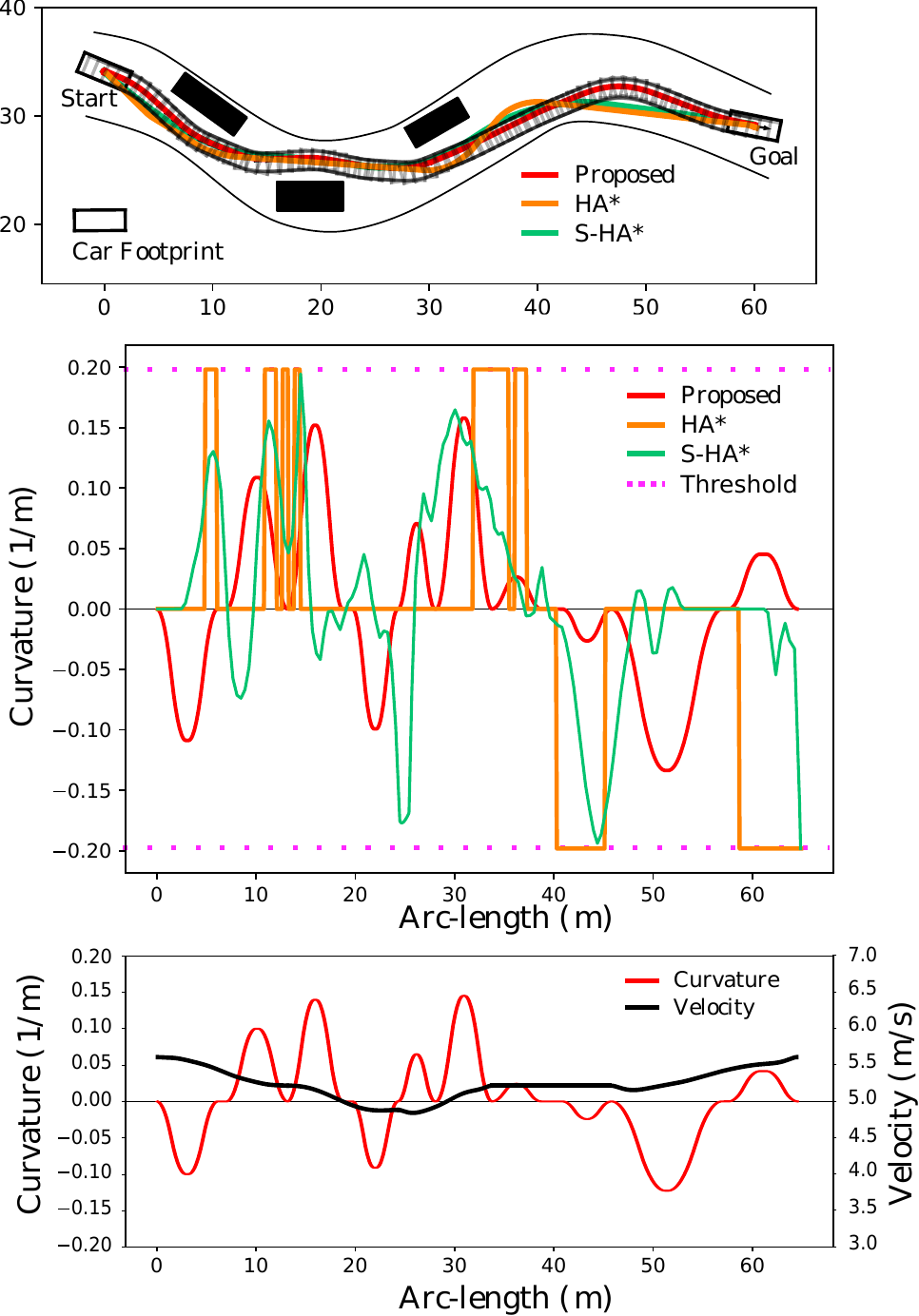}
  \caption{Result of highway maneuver with several obstacles. Top: resulting motions, mid: curvature over motions, bottom: proposed curvature and velocity profile.}
  \label{CompleteTraj}
\end{figure}

\begin{table}[htp]
\centering
\begin{tabulary}{\linewidth}{m{0.09\linewidth} p{0.08\linewidth} p{0.16\linewidth} p{0.17\linewidth} p{0.1\linewidth} p{0.10\linewidth}} 
\toprule
 Method & Length ($m$) & Smoothness Cost & Max Curvature ($m^{-1}$)  & Runtime Speedup & Velocity \\
\midrule
HA* & 65.0 & (1.35, 0.77) & 0.198 & 1.0 & No\\
S-HA* & 64.9 & (0.36, 0.46) & 0.194 & 0.6 & No\\
Proposed &64.6 & (0.17, 0.28) & 0.158 & 4.7 & Yes\\
\bottomrule\\
CGPrim & 65.8 & ( - , 0.44) &0.189 & 4.5 & No\\
\bottomrule
\end{tabulary}
\caption{Road navigation results. CGPrim from \cite{zhang2018hybrid} Table 3 for identical planning problem.}
\label{PathSave2}
\end{table}

\section{Discussion \& Conclusion}
We present a novel technique to compute an optimal set of motion primitives for use in lattice-based motion planning by way of a mixed integer linear program. Further, we propose an A*-based algorithm using these primitives to compute motions between configurations and a post-processing smoothing algorithm to remove excessive oscillations from the motions. 

The results of the previous section illustrate the effectiveness of the proposed technique. Indeed, feasible, smooth motions were computed in both parking lot and highway settings. The proposed approach results in motions with an improved level of comfort in two common metrics: integral squared jerk, and smoothness as compared to state of the art approaches. While these latter techniques may be coupled with additional smoothing techniques, this process will only increase the runtime. While short turns and straight lines as primitives (as employed in \cite{dolgov2008practical}) lend versatility in movement -- a useful feature in parking lot scenarios, they also result in larger graphs which must be traversed by a planner. A better approach resulting in smoother motions with increased runtime performance is to create longer compound actions using these basic motions, study which sequence of motions can be generated using others to within an acceptable tolerance, pre-smooth these motions (e.g. by using $G3$ curves), and store them as an action set. This is precisely the methodology employed by this work.  If the motion primitives already include velocity as a state, then a velocity profile may be easily computed. This paper has not proposed a controller to track the reference paths we compute nor does it propose a framework for re-planning. Further, the resolution for the lattices used in Section \ref{ResultsSection} was chosen based on trial and error given the experiments. Though the choice of resolution is crucial for lattice-based motion planning, this work focused primarily on lattice traversal rather than lattice generation. These are left for future work. 

\bibliographystyle{IEEEtran}
\bibliography{references}
\vspace{-1cm}
\begin{IEEEbiography}[{\includegraphics[width=1in,height=1.25in,clip,keepaspectratio]{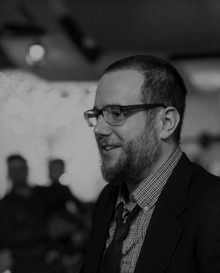}}]{Alexander Botros}
is a postdoctoral fellow in the Autonomous Systems Lab at the University of Waterloo. His research focuses primarily on local planners and trajectory generation for autonomous vehicles. In particular, Alex is researching the problem of computing minimal t-spanning sets of edges for state lattices with the goal of using these sets as motion primitives for autonomous vehicles. Alex Completed his undergraduate and M.Sc. engineering work at Concordia University in Montreal, and his Ph.D. at the University of Waterloo. 
\end{IEEEbiography}
\vspace{-1.3cm}
\begin{IEEEbiography}[{\includegraphics[width=1in,height=1.25in,clip,keepaspectratio]{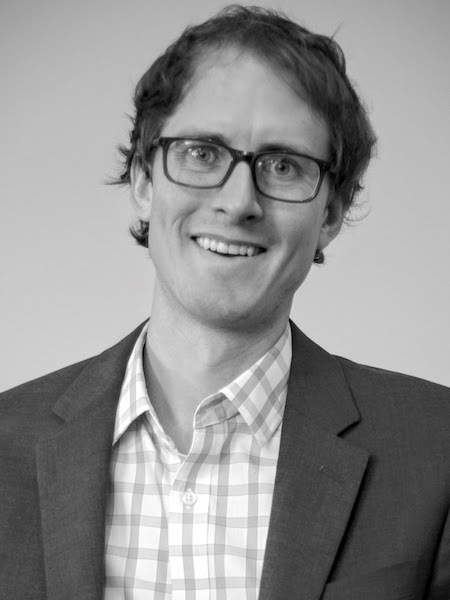}}]{Stephen L. Smith} (S'05--M'09--SM'15) received the B.Sc. degree from Queen’s University, Canada in 2003, the M.A.Sc. degree from the University of Toronto, Canada in 2005, and the Ph.D. degree from UC Santa Barbara, USA in 2009. From 2009 to 2011 he was a Postdoctoral Associate in the Computer Science and Artificial Intelligence Laboratory at MIT, USA.
He is currently a Professor in Electrical and Computer Engineering at the University of Waterloo, Canada and a Canada Research Chair in Autonomous Systems. His main research interests lie in control and optimization for autonomous systems, with an emphasis on robotic motion planning and coordination.
\end{IEEEbiography}
\end{document}